\documentclass{amsart}

\usepackage{amssymb}
\usepackage{setspace}
\usepackage{hyperref}
\usepackage{xcolor}
\usepackage[utf8]{inputenc}
\usepackage{cleveref}

\newcommand{\B}{\mathcal{B}}
\newcommand{\C}{\mathcal{C}}
\newcommand{\D}{\mathcal{D}}
\newcommand{\HH}{\mathcal{H}}
\newcommand{\I}{\mathcal{I}}
\newcommand{\E}{\mathcal{E}}
\newcommand{\G}{\mathcal{G}}
\newcommand{\m}{\mathfrak{m}}
\newcommand{\NN}{\mathcal{N}}
\newcommand{\n}{\mathfrak{n}}
\newcommand{\nn}{\mathfrak{N}}

\newcommand{\QQ}{\mathcal{Q}}

\newcommand{\N}{\mathbb{N}}
\newcommand{\R}{\mathbb{R}}
\newcommand{\Z}{\mathbb{Z}}

\renewcommand{\epsilon}{\varepsilon}

\def\DBS{\operatorname{DBSCAN}}
\def\HDBS{\operatorname{HDBSCAN}}
\def\SDBS{\operatorname{S-DBSCAN}}
\def\SHDBS{\operatorname{S-HDBSCAN}}
\def\Int{\operatorname{int}}
\def\core{\operatorname{core}}
\def\max{\operatorname{max}}

\numberwithin{equation}{section}
\newtheorem{thm}{Theorem}[section]
\newtheorem{lem}[thm]{Lemma}
\newtheorem{cor}[thm]{Corollary}
\newtheorem{prop}[thm]{Proposition}
\newtheorem{obs}[thm]{Observation}

\theoremstyle{definition}
\newtheorem{defn}[thm]{Definition}
\theoremstyle{remark}
\newtheorem{rmk}[thm]{Remark}

\begin{document}

\title{Geometric reconstructions of density based clusterings}

\author{A.\ L.\ Garcia-Pulido}
\email{A.L.Garcia-Pulido@liverpool.ac.uk}
\address{Department Of Computer Science, Ashton Building, University of Liverpool, Liverpool L69 3BX, UK}

\author{K.\ P.\ Samardzhiev}
\email{k.samardzhiev@liverpool.ac.uk}
\address{Geographic Data Science Lab, University of Liverpool, Roxby Building, Liverpool, L69 7ZT
    United Kingdom}

\thanks{ALGP is a member of the Centre for Topological Data Analysis funded by the EPSRC grant ``New Approaches to Data Science: Application Driven Topological Data Analysis'' EP/R018472/1.}

\begin{abstract}
$\DBS^*$ and $\HDBS^*$ are well established density based clustering algorithms. However, obtaining the clusters of very large datasets is infeasible, limiting their use in real world applications.

By exploiting the geometry of Euclidean space, we prove that it is possible to systematically construct the $\DBS^*$ and $\HDBS^*$ clusters of a finite $X\subset \R^n$ from specific subsets of $X$. 
We are able to control the size of these subsets and therefore our results make it possible to cluster very large datasets. 

To illustrate our theory, we cluster the Microsoft Building Footprint Database of the US, which is not possible using the standard implementations.
\end{abstract}

\maketitle

\section{Introduction}

Let $(X,d)$ be a finite metric space. A \emph{clustering} of $X$ often refers to finding disjoint subsets (\emph{clusters}) of $X$ such that points within a subset are closer to each other than to points in the others. A variant of this is density based clustering, where in addition, the points in a cluster must satisfy a given density condition. This clusters $X$ into ``high density'' regions. Density based clustering has received interest from many fields because of its applicability to a vast range of real world problems \cite{covid19clusters,imageSegmentation,influenza, alertstorm,housepriceprediction,statham2020applications}.

A widely used density based clustering algorithm is $\DBS$ \cite{Ester96}, we work with $\DBS^*$ which is a modification of the original algorithm.
Given parameters $\epsilon>0$ and $k\in\N$, $\DBS^*$ first considers the set $\C(\epsilon)\subset X$ of points $p\in X$ for which the cardinality of $B_\epsilon(p)$ at least $k+1$. It then clusters $\C(\epsilon)$ by adding $p,q\in\C(\epsilon)$ to the same cluster if $d(p,q)\leq \epsilon$ (see \Cref{sec:DBSCAN_HDBSCAN}). 
Formally, we consider the $\DBS^*$ clusters as the connected components of a graph with vertex set $\C(\epsilon)$ and $(p,q)$ an edge if $d(p,q)\leq \epsilon$ (see \Cref{absy}).

Since it is often unclear clear how to choose $\epsilon$, $\HDBS^*$ evaluates $\DBS^*$ clusters over the full range of $\epsilon>0$ and, by evaluating how this hierarchy of clusters evolves, selects the most prevalent clusters of the dataset \cite{Campello13}.
A core part of $\HDBS^*$ is the construction of this hierarchy using a complete weighted graph $(X,\omega)$, where $\omega(p,q)$ is the maximum of $d(p,q)$ and the least $\epsilon>0$ such that $p,q\in \C(\epsilon)$ (see \Cref{reach_def}).
The $\DBS^*$ clusters of parameter $\epsilon$ are then the non-trivial connected components of a graph $(X,\omega)_\epsilon$ obtained by removing all edges of weight strictly greater than $\epsilon$.

An essential component of $\DBS^*$ and $\HDBS^*$ are $k$ nearest neighbour searches, but their complexity grows quadratically and the memory requirements increase with the size of the dataset. This makes it impractical, or even impossible, to use these algorithms for cluster analysis of very large datasets.
Although there has been significant effort to improve their performance \cite{galan2019comparative,patwary2012new, he2014mr,McInnes17,Jackson2018,Wang2021,neto2017efficient,de2019faster}, the tested data and parameters are incomparable to those used in real world applications. A thorough benchmark \cite{hbscan_benchmarking} shows that, even when running overnight, it is only feasible to cluster datasets with 5 million points and with $k\leq 5$.
For comparison, the Microsoft Building Footprint Database of the US \cite{MicrosoftBuildingData} has 125 million points and one requires $k\geq 1,000$ to delineate geographical areas of interest.

In this article we introduce a new framework to obtain the $\DBS^*$ and $\HDBS^*$ clusters of very large datasets.
In both cases, we offer an alternative to the explored approaches of optimising and approximating the steps of the algorithms in order to do this. Instead, we provide constructions that avoid the lengthy $k$ nearest neighbour searches and prove that they produce the same clusters.

Our starting point is to demonstrate how it is possible to use the natural partition of Euclidean space into cubes to produce $\DBS^*$ clusters subject to minimising the number of distance calculations (see \Cref{sec:parallel_dbscan}). 
In particular, our construction allows us to skip $k$ nearest neighbour calculations in highly dense areas of $X$, which are a common challenge for implementations such as sklearn.DBSCAN. In fact, since our proof is constructive we use it to derive an algorithm $\SDBS^*$ to construct the $\DBS^*$ clusters, which is parallel and scales to very large datasets.

For $\HDBS^*$, we construct a weighted graph $(X,\omega')$ and prove that, for every $\epsilon>0$, the connected components of $(X,\omega')_\epsilon$ precisely equal those of $(X,\omega)_\epsilon$, see \Cref{thm:1}. To mitigate the asymptotic complexity of calculating $k$ nearest neighbours the main idea is to construct this graph from small independent subsets $S_i\subset X$, whilst ensuring that the connected components are preserved (see Definitions \ref{defn:good_set_reach} and \ref{defn:frankie_stages}). 

As with $\SDBS^*$, we use our proof to give an algorithm, $\SHDBS^*$, that constructs the $\HDBS^*$ clusters, see \Cref{sec:QHDBS}. 
We show that it is possible to iterate through progressively less dense areas of the dataset; partition the data of each iteration into independent, manageable pieces; remove parts of the data that are redundant for the processing of subsequent iterations; and recover the $\HDBS^*$ clusters of the full dataset by processing the small pieces of each iteration.

Reconstructing a global property of a graph relying only on local properties is a very hard task. In our case, a central factor that makes it possible to achieve this is, is our careful selection of the $S_i$ and the identification and explicit description of reduced subsets $\partial S_i$ that contain the relevant points that interact with other $S_j$.

In fact, we develop theory that allows us to study finite subsets of $\R^n$ in a similar way to how one treats open sets in classical topology, see \Cref{sec:int_bdry_sets}. This machinery enables us to incorporate Euclidean geometry onto finite sets, for example, define notions such as boundary and neighbourhood of a set; calculate distances between sets and prove that it is attained near the boundary (see \Cref{defn:int_bdry_closure} and \Cref{prop:dist_attained_near_bdry}).  This point of view is instrumental in our construction of $(X,\omega')$ and in the proof that it has the desired property.  It proves a great computational asset by controlling the number of distance calculations when clustering a real world dataset. 

To put our framework into practice, we include examples where we cluster two real world datasets, using $\SDBS^*$ in \Cref{sec:applications_QDBS}, and $\SHDBS^*$ in \Cref{sec:applications_QHDBS}. Our largest dataset is the entire Microsoft Building Footprint Database of the United States \cite{MicrosoftBuildingData} which we cluster with $k=1,900$.

\section{\texorpdfstring{$\DBS^*$}{DBSCAN*} and \texorpdfstring{$\HDBS^*$}{HDBSCAN*}}\label{sec:DBSCAN_HDBSCAN}

Given an edge-weighted graph $G$, we write $V(G)$ for the set of vertices of $G$, $E(G)$ for the set of edges of $G$ and $\omega(G)$ for the weight function $\omega(G)\colon E(G)\to [0,\infty)$. We will often omit the argument of $G$ whenever it is clear from the context. We will write $\pi_0(G)$ for the set of connected components of $G$.

Let $G,H$ be two weighted graphs. The \emph{union}  $G\cup H$ is the weighted graph with vertex set $V(G)\cup V(H)$, edge set $E(G)\cup E(H)$ and weights
\[
    \omega(p,q) = \begin{cases}\min\{\omega(G)(p,q),\omega(H)(p,q)\} & (p,q)\in E(G)\cap E(H)        \\
             \omega(G)(p,q)                        & (p,q)\in E(G)\setminus E(H)   \\
             \omega(H)(p,q)                        & (p,q) \in E(H)\setminus E(G).\end{cases}
\]

If $\epsilon \in [0,\infty)$ and $G$ is a weighted graph, the graph $G_\epsilon$ will denote the subgraph of $G$ with vertex set $V(G)$, set of edges $E(G_\epsilon) = \{(p,q)\in E(G): \omega(G)(p,q)\leq \epsilon\}$ and with weight $\omega(G_\epsilon):= \omega(G)|_{E(G_\epsilon)}$. For $k\in\N$, let $G_{\epsilon,k}$ denote the subgraph of $G_\epsilon$ induced by vertices of degree at least $k$.

Given a finite metric space $(X,d)$, let $G(X,d)$ be the complete graph on $X$ with weight $\omega(p,q) = d(p,q)$.

\begin{rmk} In the context of topological data analysis, $G(X,d)_\epsilon$ is sometimes referred to as the \textit{$\epsilon$-neighbourhood graph} of $X$ and it is the $1$-skeleton of the well-known \v Cech complex of $(X,d)$.
\end{rmk}

Let $(X,d)$ be a metric space. Given a point $p \in X$ and $\epsilon\geq 0$, we denote by $B_{\epsilon}(p)$ the closed ball of radius $\epsilon$ with centre $p$. The cardinality of a set $A$ will be denoted by $|A|$. Given $A, B\subset X$  and $p\in X$ define the \emph{distance to $A$ from $p$} as
\[
    d(p,A) = \inf_{a\in A} d(x,a),\]
and the \emph{distance between $A$ and $B$} as
\[
    d(A,B) = \inf_{a\in A} d(a,B).
\]
For $\epsilon\geq 0 $ define the \emph{$\epsilon$-neighbourhood} of $A\subset X$ as
\[
    B_\epsilon(A) =  \{p \in X: d(p,A)\leq \epsilon\}.
\]

Throughout this article we work with a fixed $k\in \N$.

In this section we describe two density based clustering algorithms: $\DBS^*$ and $\HDBS^*$.

\subsection{\texorpdfstring{$\DBS^*$}{DBSCAN*}}\label{sec:DBSCAN}

In this section, we fix $(X,d)$ a finite metric space and $\epsilon \geq 0$. 
$\DBS^*$ is an algorithm that clusters regions of $X$ with minimum local density $k+1$. The local density of $X$ is measured using the cardinality of balls with radius $\epsilon$. 

\begin{defn}\label{defn:core_noise_points}
    A point $p\in X$ is called a \emph{core point} if $B_\epsilon(p)$ contains at least $k$ points excluding $p$, that is, if
    \[
        |B_{\epsilon}(p)| > k.
    \] We denote by $\C(\epsilon)$ the set of core points and define the set of \emph{noise points} $\NN(\epsilon) = X\setminus \C(\epsilon)$.
\end{defn}

Notice that $\C(\epsilon) = V(G(X,d)_{\epsilon,k})$.

\begin{defn}\label{absy}
    The set of $\DBS^*(\epsilon)$ \emph{clusters} is
    \[
    \D^*(\epsilon) = \{ V(C):C\in \pi_0(G_{\epsilon,k})\}.    
    \] When $\epsilon$ is clear from the context, we write $\C$, $\NN$ and $\D^*$ instead of $\C(\epsilon)$, $\NN(\epsilon)$ and $\D^*(\epsilon)$, respectively.
\end{defn}

\subsection{\texorpdfstring{$\HDBS^*$}{HDBSCAN*}}\label{sec:HDBSCAN}

$\HDBS^*$ evaluates $\DBS^*(\epsilon)$ clusters over the full range $\epsilon \in [0,\infty)$. By considering how these clusters evolve and persist as $\epsilon$ varies, it is possible to quantify the relative density of a cluster compared to surrounding regions.
The final $\HDBS^*$ clusters are $\DBS^*$ clusters which prevail the most through all scales. These clusters may come from different $\DBS^*$ scales.

We recall the construction of $\HDBS^*$ and its persistence score, described as in \cite{McInnes17}. 
Again, for this section, we fix $(X,d)$ a finite metric space.

\begin{defn}\label{reach_def}
    A \emph{nearest neighbour} of $p\in X$ is an element $p_1\in X$ such that $d(p,p_1) = d(p,X\setminus\{p\})$.
    For $j > 1$, a \emph{$j$-th nearest neighbour} of $p$ is an element $p_j\in X$ which satisfies $d(p,p_j) = d(p,X\setminus\{p,p_1,\cdots,p_{j-1}\})$, where $p_i$ is an $i$-th nearest neighbour for $1\leq i \leq j-1$.
    For $p\in X$, we will write $\core_k(p)$ for the distance from $p$ to a $k$-th nearest neighbour.

    We define the \emph{reachability distance} $\rho \colon X\times X\to \R$ as
    \[
        \rho(p,q) = \max\{\core_k(p),\core_k(q), d(p,q)\}
    \] if $p\neq q$ and $\rho(p,p) = 0$.
\end{defn}

To describe $\HDBS^*$ we first note that for any finite metric space $(X, d)$ and $\epsilon >0$, 
\[
    E(G(X,\rho)_\epsilon) = E(G(X,d)_{\epsilon, k}).
\] Consequently, for $p\neq q \in X$,
\begin{equation}\label{eqn:dbs_components_equal_hdbs_components}
p,q\in\Gamma\text{ for some } \Gamma \in \pi_0(G(X,\rho)_\epsilon) \Leftrightarrow    p,q\in\tilde{\Gamma}\text{ for some } \tilde{\Gamma} \in \pi_0(G(X,d)_{\epsilon, k}).  
\end{equation} Thus, by varying $\epsilon$ from $0$ to $\operatorname{diam}(X)$, $G(X,\rho)_\epsilon$ produces a hierarchy of the $\DBS^*$ clusters. $\HDBS^*$ first derives a summary of this hierarchy by 
defining an equivalence relation on these clusters.

Let $\epsilon\geq 0$. For $0\leq\gamma\leq\epsilon$, the inclusion 
\[
    i_{\gamma} \colon G(X,d)_{\gamma, k}\hookrightarrow G(X,d)_{\epsilon, k},
\] gives the induced map
\[
    (i_{\gamma})_* \colon \pi_0(G(X,d)_{\gamma, k})\hookrightarrow \pi_0(G(X,d)_{\epsilon, k}).
\] For $m\in \N$, set
\[
\D^*(\epsilon, m) = \{C \in \D^*(\epsilon): |C|\geq m \}.
\] Define an equivalence relation on
\[
\bigcup_{\epsilon\in [0,\infty)} \D^*(\epsilon,m)    
\] by $C\sim D$, with $C\in \D^*(\delta,m)$ and $D\in \D^*(\epsilon, m)$ and without loss of generality $\delta\leq\epsilon$, if and only if
\[
(i_\delta)_*(C) = D   \text{ and } |(i_\gamma)_*^{-1}| = 1 
\] for every $\gamma\in[\delta,\epsilon]$. 

Notice that for an equivalence class $[C]$ and $\gamma\in[0,\infty)$
\[ 
|[C]\cap \D^*(\gamma,m)| \in \{0,1\}.
\] Set $C_\gamma$ to be the unique element of $[C]\cap \D^*(\gamma,m)$ if the latter is non empty, and $\emptyset$ otherwise. Define the \emph{persistence score} of an equivalence class $[C]$ by
\[
\Sigma([C]) =\int_{0}^{\infty} \frac{|C_\gamma|}{\gamma^2}.
\] This is a weighted average of the density of the clusters $C_\gamma$ within an equivalence class.

The final step in $\HDBS^*$ is to choose final clusters amongst the equivalence classes. For an equivalence class $[C]$, define
\[
p([C]) = \bigcup_{\epsilon\geq 0} C_\epsilon.    
\] Note that since $X$ is a finite metric space there are only a finite number of equivalence classes, say $\{[C^1],\ldots,[C^\alpha]\}$. We now select $I\subset\{1,\ldots, \alpha\}$ which maximises
\[
\sum_{i\in J} \Sigma [C^i]    
\] amongst all $J\subset\{1,\ldots, \alpha\}$ with
\[
    p([C^i])\cap p([C^j]) = \emptyset
\] whenever $i\neq j\in J$.
The set of $\HDBS^*$ \emph{clusters} is 
\[
   \HH^* =  \{ p([C^i]): i\in I \}.
\]

\section{Reconstructing \texorpdfstring{$\DBS^*$}{DBSCAN*} using covering cubes}\label{sec:parallel_dbscan}
In this section we describe $\SDBS^*$, our parallel algorithm that produces $\DBS^*$ clusters. This algorithm relies on the natural partition of Euclidean space into cubes.

From now on, $X$ will denote a finite subset of $\R^n$ with $n\geq 2$. We will consider subsets of $\R^n$ equipped with either the $l_2$ metric, denoted by $d$, or the reachability distance derived from $d$, denoted by $\rho$. Whilst our ideas are applicable to any norm on $\R^n$ we restrict to this setting for simplicity. 

\begin{defn}\label{defn:cubes}
For $\epsilon > 0$, define $\QQ(\epsilon)$ to be the collection of cubes of the form
\[
\left\{(x_1,\ldots,x_n)\in \R^n:\quad j_i\frac{\epsilon}{2\sqrt{n}} \leq x_i \leq (j_i+1)\frac{\epsilon}{2\sqrt{n}}\right\},
\] with $j \in \Z^n$. 

We say that $S,T\in\QQ(\epsilon)$ are \emph{adjacent} if $S\cap T\neq\emptyset$. For $S\in\QQ(\epsilon)$ and $m\in \N\cup \{0\}$, define the $m$-extension of $S$ by
\[
S^m=\left\{x\in \R^n: \max{|x_i-s_i|}\leq m\frac{\epsilon}{2\sqrt{n}}\text{ for some }s\in S\right\}.    
\]

Let $\m$ be the smallest integer such that $\m\geq 2\sqrt{n}$. Note that $S^{\m}$ is the smallest extension of $S$ satisfying $B_\epsilon(S)\subset S^\m$.

For $A\subset \R^n$, define 
\[
\I(A) = \{S\in\QQ(\epsilon) : S\cap A\neq \emptyset \}. 
\] For $Y\subset \R^n$ we write $Y_A$ for $Y\cap A$. 
\end{defn}

For the rest of the section we fix $\epsilon>0$. Recall that $\C$ are the core points of $\DBS^*$ from \Cref{defn:core_noise_points}.

Note that, because of the choice of side length of cubes $S\in\QQ(\epsilon)$,
\begin{equation*}
    p\in S,\ q\in S^1 \Longrightarrow d(p,q)\leq \epsilon.
\end{equation*}  In particular, 
\begin{equation}\label{eqn:first_ext_same_cluster}
p,q\in S_{\C} \Longrightarrow p,q\in C
\text{ for some } C\in \D^*.
\end{equation}

$\SDBS^*$ constructs a graph $\G= (\I(\C),\E)$ whose connected components are in bijective correspondence with the $\DBS^*$ clusters: $S,T\in \I(\C)$ are in the same connected component of $\G$ if and only if all of the elements of $S_\C$ and $T_\C$ belong to the same $\DBS^*$ cluster.
This relies on the observation in \cref{eqn:first_ext_same_cluster}.

The idea behind the construction of this graph arises from algebraic topology. In terms of covering spaces, a $\DBS^*$ cluster is equivalent to a connected component of $\cup\mathfrak{B}$, where one regards 
\[
   \mathfrak{B} = \{B_{\epsilon/2}(x)\}_{x\in \C}
\] as a covering of $\C$ in $\R^n$.  
We consider two other coverings of $\C$, $\mathfrak{U}=\I(\C)$ (a refinement of $\mathfrak{B}$) and $\mathfrak{W}=\{S^\mathfrak{l}\}_{S\in\I(\C)}$ (refined by $\mathfrak{B}$), with $\mathfrak{l}$ the smallest integer $\mathfrak{l}\geq \sqrt{n}$. 
The connected components of $\cup\mathfrak{U}$ and $\cup\mathfrak{W}$ are an under and over approximation of the connected components of $\cup\mathfrak{B}$. In general, both of $\mathfrak{U}$ and $\mathfrak{W}$ are significantly smaller than $\C$ and $\mathfrak{B}$, especially around highly dense areas of $\C$. Moreover, working with a regular tiling by cubes is more efficient than working with arbitrary balls.

The first step is to identify the core points $\C$ and the cubes containing them $\I(\C)$. For this we:
\begin{itemize}
\item Categorise each cube $S\in\QQ(\epsilon)$ as ``dense", ``sparse" or ``locally dense" depending on the cardinality of a neighbourhood of $S$. 
This is done so that ``dense" cubes only contain core points and ``sparse" cubes only contain noise points.

\item For ``locally dense" $S\in\QQ(\epsilon)$, identify the core points $p\in S$, obtain their $\epsilon$-neighbours $B_\epsilon(p)\cap X$, and identify the cubes containing any $\epsilon$-neighbours of points in $S_\C$. 
\end{itemize}

In the second step we build the graph $\G$: 
\begin{itemize}
    \item Begin with $\G_0 = (\I(\C),\E_0)$, where $(S,T)\in\E_0$ if $S\cap T\neq\emptyset$; that is, the connected components of $\G_0$ correspond to those of $\cup\mathfrak{U}$. We use the results from the first step and the cubes from $\mathfrak{W}$ to remove and add edges to $\G_0$ and obtain $\G$.
\end{itemize}

This approach, and our specific choice of side length of the cubes in $\QQ(\epsilon)$, allows us to identify large portions of $X$ for which the $\DBS^*$ clusters can be partially constructed without the need for any pairwise distance calculations.

The first step only requires us to consider a neighbourhood of a cube and can be computed in parallel. Most of the second step only requires either the neighbourhood of a cube, or a connected component of a graph, and each can be processed in parallel.

Note for computational reasons it is more convenient to use the covers $\{S^1\}_{S\in\I(\C)}$ and $\{S^\mathfrak{\m}\}_{S\in\I(\C)}$, instead of $\mathfrak{U}$ and $\mathfrak{W}$, which we do below.

\subsection{Categorising cubes}

\begin{obs}\label{obs:dense_cubes}
Let $S\in\QQ(\epsilon)$. If $|S_X^1| > k$, then $S_X \subset \C$. If $|S^{\m}_X| \leq k$ or $S_X=\emptyset$, then $S_X\subset\NN$.        
\end{obs}

\begin{proof}
    The first assertion follows from the fact that if $S$ and $T$ are adjacent cubes in $\QQ(\epsilon)$, then $d(p,q)\leq \epsilon$ for any $p\in S$ and $q\in T$. The second is satisfied since, for any $p\in S$, $B_\epsilon(p)\subset S^{\m}$.
\end{proof}

This observation motivates the following definition:
\begin{defn}\label{defn:cube_categories}
    We say that $S\in\QQ(\epsilon)$ is
    \begin{enumerate}
        \item\label{defn:dense} \emph{dense} if $S_X \neq \emptyset$ and $|S^1_X| > k$,
        \item\label{defn:sparse} \emph{sparse} if $|S^{\m}_X| \leq k$ or $S_X=\emptyset$,
        \item\label{defn:locally_dense} \emph{locally dense} if $S_X \neq \emptyset$, $|S^1_X| < k$ and $|S^{\m}_X| > k$.
    \end{enumerate}
\end{defn}

By \Cref{obs:dense_cubes}, we only have to determine the core points of locally dense cubes, which significantly reduces the number of pairwise distance calculations.

\subsection{Identifying core points in locally dense cubes and their \texorpdfstring{$\epsilon$}{E}-neighbours}

We use the following observation to identify the core points in a locally dense cube and their $\epsilon$-neighbours.
\begin{obs}\label{obs:locally_dense}
    Let $S\in\QQ(\epsilon)$ be locally dense. A $p\in S$ is a core point if and only if 
    \[
        |B_\epsilon(p)\cap (S_X^{\m}\setminus S^1)|> k-|S^1_X|.
    \]
\end{obs}
This implies that for every locally dense cube $S\in\QQ(\epsilon)$ and $p\in S$ we only need to calculate $d(p,q)$ for $q\in S^{\m}\setminus S^1$. If we deduce that $p\in S_\C$, we store these $\epsilon$-neighbours of $p$ in a list $\B_p$ and add the elements of $S^1_X\setminus\{p\}$. We also store a list $\B_S$ of the cubes $T$ such that $T_X\cap B_\epsilon(S_\C)\neq \emptyset$.

\subsection{Building the graph}

Let us first consider the graph $\G_1 = (\I(\C),\E_1)$ with $(S,T)\in \E_1$ if
\begin{enumerate}
\item\label{building_1} $S,T$ are adjacent or,
\item\label{building_2} $S,T$ are locally dense, $T\in \B_S$ and $S\in \B_T$ or
\item\label{building_3} $S$ is locally dense, $T$ is dense and $T\in\B_S$.
\end{enumerate}
This graph is easy to construct using the previous results of this section. However, $\pi_0(\G_1)$ only approximates the $\DBS^*$ clusters. 

The first problem is if $S,T\in \I(\C)$ are as in \Cref{building_2}, we cannot guarantee the existence of $p\in S_\C$ and $q\in T_\C$ with $d(p,q)\leq \epsilon$. Compare this to \Cref{building_1,building_3}, where finding such $p,q$ is guaranteed. We prune $\G_1$ to correct this.

Precisely, we calculate $\pi_0(\G_1)$ and for each $\Gamma\in \pi_0(\G_1)$ we remove edges $(S,T)$ with $S,T$ non adjacent and such that for every $p\in S$,
\[
    B_\epsilon(p)\cap T_\C = \emptyset.
\] We do this in parallel across the connected components of $\G_1$ to construct $\G_2 = (\I(\C),\E_2)$. By construction, if $(S,T)\in \E_2$, then there are $p\in S_\C$ and $q\in T_\C$ with $d(p,q)\leq \epsilon$. In particular, $p,q\in C$ for some $C\in \D^*$. By \cref{eqn:first_ext_same_cluster}, $S_\C\cup T_\C\subset C$. 

Conversely, suppose that $p\in S_C$ and $q\in T_\C$ with $d(p,q)\leq\epsilon$. If $S,T$ are adjacent or one of them is locally dense, then by definition $(S,T)\in \E_1$ and this edge is not pruned when constructing $\E_2$.

The remaining problem is if $S,T\in \I(\C)$ are both dense and belong to different connected components of $\G_2$. For every pair of such cubes with $T \subset S^{\m}\setminus S^1$, we calculate pairwise distances to determine if there are $p\in S_X\subset \C$ and $q\in T_X\subset \C$ with $d(p,q)\leq \epsilon$. If so, we add $(S,T)$ to form a final graph $\G$ and stop processing the connected components containing $S$ and $T$.

We handle this case last to avoid as many pairwise distance calculations involving dense cubes as possible; this is crucial as their cardinality can be very large. By doing this step last we increase the likeliness that two dense cubes will be connected by a path constructed earlier in the process.

\section{Geometric constructions relative to covering cubes}\label{sec:int_bdry_sets}

In this section, we fix $\epsilon>0$. Recall from \Cref{defn:cubes} the set of cubes $\QQ=\QQ(\epsilon)$ and its properties. We also fix $A\subset X$ and recall that $\I(A) =\{S\text{ cube}: S_A\neq\emptyset\}$.
We use the elements of $\QQ$ to define standard notions from topology (interior, boundary and closure) of $A$ relative to $\QQ$.
These constructions are a fundamental component of our alternate construction of $\HDBS^*$ clusters.

At the end of the section we demonstrate how these constructions allow us to easily extend $\SDBS^*$ to $\SDBS$, a parallel algorithm that obtains the $\DBS$ clusters, whilst minimising pairwise distance calculations. 

\begin{defn}\label{defn:int_bdry_closure}
A cube $S\in \I(A)$ is an \emph{interior cube} of $A$ if $S^1_X \subset A$ and $T\in\I(A)$ for every cube $T\subset S^1$. The \emph{interior of $A$} is the union of interior cubes,
\[
\Int(A) = \bigcup_{S\text{ interior}} S.
\]

A cube $S\in \I(A)$ is a \emph{boundary cube} of $A$ if $S$ is not an interior cube.
 The \emph{boundary of $A$} is defined as
\[
    \partial A = \bigcup\limits_{S\text{ boundary}} S,
\] and the \emph{closure of $A$} is
\[
    \bar{A} = \Int(A)\cup \partial A= \bigcup\limits_{S\in \I(A)} S.
\]

Let $Y \subset \R^n$ and $N\geq 0$. We define the \emph{$N$-extension} of $Y$ to be
\[
    Y^N = \bigcup\limits_{S\in \I(Y)} S^N
\]
\end{defn}

From now on we will write $\n$ for the smallest integer satisfying $\n\geq \sqrt{n} - 1$. Notice that $\n\geq 1$.

We now show that the distance to a point outside $A$ is attained near the boundary of $A$. 
\begin{prop}\label{prop:dist_attained_near_bdry}
    Let $x\in X\setminus A$ and $r\in\N\cup\{0\}$. For any $p\in A\setminus(\partial A)^{\n+r}_A$, there exists $a\in (\partial A)_A$ such that
    \[
        d(p,x) > d(a,x) + r\frac{\epsilon}{2\sqrt{n}}.
    \]
    In particular, for any $x\in X\setminus A$ and $p\in A$, there exists $a\in (\partial A)^\n_A$ such that 
    \[
    d(a,x)\leq d(p,x).    
    \]
\end{prop}

\begin{proof}
Let $x,p$ be as in the hypothesis and $S\in\QQ$ with $p\in S$.

First note that
\begin{equation}\label{eqn:extension_in_interior}
    S^{\n+r}\subset \Int(A).
\end{equation} Indeed $S\subset \Int(A)$ because $S\in\I(A)$ and $p\not\in\partial A$. If $i \geq 0$ is maximal such that $S^i\subset\Int(A)$ and $i < \n +r$, there exists $T\in \QQ$ with $T\subset S^{i+1}\cap \partial A$ and this would imply
\[
S\subset T^{i+1}\subset (\partial A)^{i+1}\subset (\partial A)^{\n+r},    
\] contradicting the hypothesis.
Note that \cref{eqn:extension_in_interior} implies
\begin{equation}\label{eqn:last_extension}
 T \in\QQ,\ T\subset S^{\n+r+1} \Longrightarrow T\in\I(A) \text{ and } T\not\in\I(X\setminus A).
\end{equation}

Let $[p,x]$ denote the line segment from $p$ to $x$ and let $0=t_0 < t_1 < t_2<\ldots< t_j=1$ and $p_i = (1-t_i) p + t_i x$ be such that 
\[
[p_{i-1},p_{i}] \in\{ S\cap [p,x]: S\in \I([p,x])\}.
\] We write $Q_i$ for the cube satisfying $[p_{i-1},p_{i}] = Q_i\cap [p,x]$.

Let $i$ be maximal such that $Q_i\in\I(A)$ and $Q_i\not\in\I(X\setminus A)$. In particular, $Q_i\subset\partial A$ and by \cref{eqn:last_extension} $Q_{i+1}\not\subset S^{\n + r + 1}$. Since $p\not\in(\partial A)^{\n+r}\supset (Q_i)^{\n+r}$ we have the strict inequality 
\begin{equation}\label{eqn:zxe}
    d(p,p_i) > (\n + r + 1 )\frac{\epsilon}{2\sqrt{n}}.
\end{equation} 

Let $a\in (\partial A)_A\cap Q_i$, so that
\begin{equation}\label{eqn:points_same_cube}
    d(p_i,a) \leq \frac{\epsilon}{2}.
\end{equation} Since $p,p_i,x$ all lie in the same line segment, we have
\[
    d(p,x) = d(p,p_i) + d(p_i,x).
\] Combining this and \cref{eqn:points_same_cube,eqn:zxe} and using the reverse triangle inequality,
\begin{align*}
    d(p,x) &> (\n+r + 1)\frac{\epsilon}{2\sqrt{n}} + d(x,a) - d(a,p_i)\\
&\geq (\n+r + 1)\frac{\epsilon}{2\sqrt{n}} + d(x,a) -\frac{\epsilon}{2}\\
&\geq r\frac{\epsilon}{2\sqrt{n}} + d(x,a),
\end{align*} using the fact that $\n+1\geq \sqrt{n}$.
\end{proof}

We next show that in order to find the closure of the extension of a set, we only need to extend its boundary.
\begin{lem}\label{lem:ext_closure_equal_ext_bdry}
    If $r\in \N\cup \{0\}$, then $(\bar{A})^r = \Int(A)\cup (\partial A)^r.$
\end{lem}

\begin{proof}
    Certainly $(\bar{A})^r \supset \Int(A)\cup (\partial A)^r.$
    To prove the other containment, let $S\in \I(A)$. If $S\subset\partial A$, then $S^r \subset (\partial A)^r$. Otherwise, suppose $S\subset\Int(A)$ and let $T\in\QQ$ be such that $T\subset S^r$. Notice that if $T\in \I(A)$, then $T\subset \bar{A}\subset \Int(A)\cup(\partial A)^r$. 
    
    Now suppose $T\not\in\I(A)$. Then there exists $t>0$ such that $T^{t-1}_A =\emptyset$ and $T^{t}_A \neq\emptyset$.  
    Since $S\subset T^r$ and $S\in\I(A)$, $t<r$. Let $Q\subset T^t$ be such that $Q\in\I(A)$. Then $Q\subset\partial A$ and so $T\subset (\partial A)^t\subset (\partial A)^r$.
\end{proof}

Finally, we note that to find the $\epsilon$-neighbours of $A$ it is sufficient to extend the boundary $\m$ times.
\begin{cor}\label{obs:epsilon_neighood_extension}
    \[
        B_\epsilon(A)\subset \bar{A}^{\m}=\Int(A)\cup (\partial A)^{\m}.
    \]
\end{cor}

\begin{proof}
    For any $S\in \I(A)$, $B_\epsilon(S) \subset S^{\m}$ and so $S^{\m}\subset \bar{A}^{\m}$. In consequence,
    \[
        B_\epsilon(A) \subset \bigcup\limits_{S\in\I(A)}S^{\m}\subset\bar{A}^{\m}.
    \]
\end{proof}

\subsection{\texorpdfstring{$\SDBS$}{S-DBSCAN}}
The original density-based clustering algorithm $\DBS$ was first introduced in \cite{Ester96}. In contrast to $\DBS^*$, $\DBS$ makes a distinction between non-core points which are $\epsilon$-neighbours of core points and those that are not.

\begin{defn}
    An $x\in \NN(\epsilon)$ is a \emph{border point} if $x\in B_\epsilon(C)$ for some $C\in \D^*(\epsilon)$. $\DBS$ assigns any border point $x$ to the first cluster $C\in \D^*(\epsilon)$ that the algorithm finds such that $x\in B_\epsilon(C)$. We denote the set of $\DBS(\epsilon)$ clusters by $\D(\epsilon)$.
\end{defn}

Using $\SDBS^*$ and the results from \Cref{sec:int_bdry_sets}, it is straightforward to obtain $\DBS$ clusters.

\begin{rmk}\label{rmk:DBSCAN}
Fix $C\in \D^*$. By \Cref{obs:epsilon_neighood_extension}, the border points of $C$ are contained in $(\partial C)^{\m}_{\NN}$. By \Cref{prop:dist_attained_near_bdry}, $x\in (\partial C)^{\m}_{\NN}$ is a border point of $C$ if there exists $p\in (\partial C)_C^{\n}$ such that $d(p,x)\leq \epsilon$. For locally dense $S\in\QQ$, we previously calculated the $\epsilon$-neighbours of $p\in S$ and so we can immediately add those that belong to $\NN$. For dense cubes $S\in\QQ$, we can immediately add $S^1_\NN$ to $C$. Thus we only need to restrict to  pairwise distance calculations between core points in a dense cube $S\subset(\partial C)^\n$ and noise points in $T\subset S^{\m}\setminus S^1$.
\end{rmk}

\section{Reconstructing \texorpdfstring{$\HDBS^*$}{HDBSCAN*} using cluster  boundaries}\label{sec:QHDBS}

In this section we use the theory from \Cref{sec:int_bdry_sets} to give an alternate construction to obtain the $\HDBS^*$ clusters of $X$.
In fact, we construct a weighted graph $\mathcal{F}$ satisfying
\begin{equation}\label{eqn:frankie_compatible_hdbscan}
    \{V(C):C\in\pi_0(\mathcal{F}_\alpha)\} = \{V(C):C \in\pi_0 (G(X,\rho)_\alpha)\},
\end{equation} for every $\alpha \geq 0$. The definition of the persistence score and the process used to choose the final clusters from \Cref{sec:HDBSCAN} can be applied to $\mathcal{F}$ to produce the $\HDBS^*$ clusters. Our proofs are constructive and we use them in \Cref{sec:applications_QHDBS} to describe our algorithm $\SHDBS^*$.
 
We construct $\mathcal{F}$ by constructing  smaller weighted graphs restricted to subsets $C_i\subset X$ and combine them to form $\mathcal{F}$. 
Since our central requirement is that $\mathcal{F}$ satisfies \cref{eqn:frankie_compatible_hdbscan} and that its simple construction is computationally feasible for very large $X$, we must ensure that
\begin{enumerate}
    \item\label{enu:small_ext_KNNs} we can prescribe a set $N_k(C_i)$ (the search space for $k$ nearest neighbours) that is simple, comparable in size to $C_i$ and contains the $k$-nearest neighbours of elements in $C_i$;
    \item\label{enu:indep_min_interacion} we can independently construct graphs for each $C_i$ and have minimal interaction between different $C_i$;
    \item\label{enu:sets_to_glue} we can prescribe reduced subsets of $C_i$ that contain the relevant points that interact with other $C_j$.
\end{enumerate} 

In order to meet the above constraints we 
\begin{itemize}
    \item select sets $C_i$ with a uniform bound  $\core_k\leq\epsilon$ so that taking $N_k(C_i)$ to be an extension of $C_i$ (as in \Cref{sec:int_bdry_sets}); that is, $C_i\subset\C(\epsilon)$ \cref{enu:small_ext_KNNs};
    \item iteratively construct our graph using an increasing sequence of $\epsilon$ and choose the $C_i$ at each iteration so that the $N_k(C_i)$ do not intersect other $C_j$ \cref{enu:indep_min_interacion}.
    \item ensure $C_i$ is a union of $\D^*(\epsilon)$ clusters.  Otherwise, since all the elements in a $\D^*(\epsilon)$ cluster are in the same connected component at scale $\epsilon$, we would need to search other $C_j$ for edges of weight in $[0,\epsilon]$ \cref{enu:indep_min_interacion}.
    \item ensure $C_i$ \emph{is} a $\D^*(\epsilon)$ cluster. Otherwise, $C_i$ could be dispersed throughout $X$ and its extension could be the whole of $X$ \cref{enu:indep_min_interacion}.
    \item prove that an extension of $\partial C_i$, for each $C_i$ in an iteration, contains all the points that interact with previous iterations. By construction, the $C_i$ of one iteration do not interact with each other \cref{enu:sets_to_glue}.
\end{itemize}

Having established that $\mathcal{F}$ must be  constructed iteratively using $\DBS^*$ clusters to partition $X$, we begin by considering an initial clustering of $X$, $\D^*(\epsilon_1)$, and obtain the graph $G(C,\rho)$ for every cluster $C\in\D^*(\epsilon_1)$. In practice, one picks $\epsilon_1$ so that calculating $\rho$ in each cluster is feasible.

We then wish to repeat this for a choice of $\epsilon_2 > \epsilon_1$. However, increasing the scale can result in clusters that are too large for local calculations to be efficient or feasible. This would also not take into account the graphs obtained in the initial step. Instead, we consider a subset of $X$ consisting of $\NN(\epsilon_1)$ and an extension of the boundary of each cluster in $\D^*(\epsilon_1)$. 

The key result is to show that the core points that play a role in forming $\DBS^*$ clusters at \textbf{any} subsequent scale $\alpha\geq \epsilon_1$ (i.e. points that merge two clusters or that connect a cluster to a noise point) lie near the boundary of clusters in $\D^*(\epsilon_1)$. Indeed, in \Cref{cor:dist_to_cluster_attained_near_bdry} we determine an explicit size of extension of the boundary (depending only upon $n$) which always contains these points. 

We define the subset $X_2$ containing this extension of each cluster in $\D^*(\epsilon_1)$ and $\NN(\epsilon_1)$ of $X$, see \Cref{defn:good_set_reach}. We make sure that 
the reachability distance \emph{with respect to} $X_2$ coincides with that of $X$ in all points that play a role in forming future $\DBS^*$ clusters, see \Cref{lem:local_core_reach}. We then cluster $X_2$ using $\epsilon_2$ and obtain a graph $G(C,\rho)$ for each cluster $C\in\D^*(\epsilon_2)$. These graphs are then combined with the graphs of the initial step. 

This process can be repeated for any increasing sequence of $\epsilon_i$, see \Cref{defn:frankie_stages}.
The main inductive argument is \Cref{prop:successive_reach_dist}.
To complete the hierarchical clustering, we define $\mathcal F$ by combining the constructed graph with the subgraph of $G(X_{i+1},\rho)$ with edge set
\[
\{(p,q)\in E(G(X_{i+1},\rho)): \omega(p,q)\geq \epsilon_i\},
\] see \Cref{defn:final_frankie}. In \Cref{thm:1}, we prove that \cref{eqn:frankie_compatible_hdbscan} is satisfied.

\begin{rmk}
One could also construct approximations of the $\HDBS^*$ clusters by choosing to terminate the construction of $\mathcal{F}$ at iteration $i$ based on a given condition and discard the remaining noise points. For example, we could terminate the core points exceed a given cardinality threshold, or simply skip the last step for a large $\epsilon_i$. The discarded points will have little effect on the overall clustering because they are the least dense and hence contribute least to the score.
\end{rmk}

Since we work with subsets of $X$ clustered using different values of $\epsilon$, we first introduce some notation to accommodate this. 

For $Y\subset \R^n$ and $\epsilon\geq 0$, we will write $\D^*(Y,\epsilon)$ for the set of $\DBS^*$ clusters of the metric space $(Y,d)$ and write $\C(Y,\epsilon)$ and $\NN(Y,\epsilon)$ for the corresponding sets of core and noise points, respectively.
For $p,q\in Y$, we write $\core^{Y}_k(p)$ and $\rho^Y(p,q)$ for $\core_k(p)$ and $\rho(p,q)$ relative to the metric space $(Y,d)$. Notice $\core^{X}_k(p)\leq \core^{Y}_k(p)$ and $\rho^X(p,q)\leq \rho^Y(p,q)$.

From now on, let $\nn = \n +\m$. Recall that $\n$ and $\m$ are the least integers such that $\n\geq\sqrt{n}-1$ and $\m\geq 2\sqrt{n}$ as defined in \Cref{sec:int_bdry_sets} and \Cref{defn:cubes}, respectively.

The following is an immediate corollary of \Cref{prop:dist_attained_near_bdry}.
\begin{cor}\label{cor:dist_to_cluster_attained_near_bdry}
    Let $\epsilon > 0 $, $Y\subset X$ and  $C\in\D^*(Y,\epsilon)$. For any $p\in C$ and $x\in Y\setminus C$, there exists $a\in (\partial C)^\n_C$ such that 
    \[
    d(a,x)\leq d(p,x).    
    \]
\end{cor}

Next we identify a set for which the relative reachability distance coincides with the global reachability distance for all points that are found using the previous corollary and all noise points.

\begin{defn}\label{defn:good_set_reach}
    For $\epsilon\geq 0$ and $Y\subset X$, define
\[
    J(Y,\epsilon) = \bigcup\limits_{C\in\D^*(Y,\epsilon)}(\partial C)_C^{\n}\bigcup \NN(Y,\epsilon)
\] and
\[
    F(Y,\epsilon) = \bigcup\limits_{C\in\D^*(Y,\epsilon)}(\partial C)_C^{\nn}\bigcup \NN(Y,\epsilon).
\]     
\end{defn}

\begin{lem}\label{lem:local_core_reach}
    Let $Y\subset X$ and $\epsilon >0$. If $p,q \in J(Y,\epsilon)$, then 
    \[ \rho^{F(Y,\epsilon)}(p,q) = \rho^Y(p,q).
\]

\end{lem}

\begin{proof}
We show that for every $p\in J(Y,\epsilon)$ and $1\leq i\leq k$, any $i$-th nearest neighbour of $p$ in $Y$ is contained in $F(Y,\epsilon)$, from which the result follows.

First suppose that $p\in (\partial C)^{\n}_C$ for some $C\in \D^*(Y,\epsilon)$. We will show that 
\begin{equation}\label{eqn:interior_points_further_than_epsilon}
B_\epsilon(p)_Y\subset F(Y,\epsilon).
\end{equation} Indeed let $y\in B_\epsilon(p)_Y$, then
\[
\max_{1\leq j\leq n}{|y_j-p_j|}\leq d(y,p)\leq \epsilon \leq \frac{\m}{2\sqrt{n}}\epsilon.    
\] Since $p\in(\partial C)^\n$, there exists $x\in \partial C$ with 
\[
\max_{1\leq j\leq n}{|x_j-p_j|}\leq \frac{\n}{2\sqrt{n}}\epsilon.    
\] By the triangle inequality,
\[
\max_{1\leq j\leq n}{|y_j-x_j|}\leq \frac{\nn}{2\sqrt{n}}\epsilon,    
\] and therefore $y\in(\partial C)^\nn\subset F(Y,\epsilon)$, proving \cref{eqn:interior_points_further_than_epsilon}. Since $p\in\C(Y,\epsilon)$, any $i$-th nearest neighbour of $p$ lies in $B_\epsilon(p)$.

Now suppose that $p\in \NN(Y,\epsilon)$. Let $y\in Y\setminus F(Y,\epsilon)$, say $y\in C\setminus (\partial C)^{\nn}_C$ for some $C\in\D^*(Y,\epsilon)$.
Let $a\in (\partial C)_C$ be given by \Cref{prop:dist_attained_near_bdry} with \[r = \nn - \n = \m \geq 2\sqrt{n}\] such that
\[
    d(p,y) > d(p,a) + \m\frac{\epsilon}{2\sqrt{n}}\geq d(p,a) + \epsilon .  
\] If $a_i$ is an $i$-th nearest neighbour of $a$, for some $1\leq i\leq k$, then
\[
d(p,y)> d(p,a) + d(a,a_i)\geq d(p,a_i). 
\]
By \cref{eqn:interior_points_further_than_epsilon} applied to $a$, we have $a_i\in F(Y,\epsilon)$. That is, for any $y\in Y\setminus F(Y,\epsilon)$, we have found $k$ closer points to $p$ that lie in $F(Y,\epsilon)$, as required.
\end{proof}

We next reformulate \Cref{cor:dist_to_cluster_attained_near_bdry} in terms of the reachability distance. 
\begin{cor}\label{cor:reach_dist_to_cluster_attained_near_bdry}
    Let $Y\subset X$ and $\epsilon>0$. Suppose $p\in C$ for some $C\in \D^*(Y,\epsilon)$ and $x\in Y\setminus C$. There exists $a\in (\partial C)_C^\n\subset J(Y,\epsilon)$ such that 
    \[
    \rho^Y(a,x)\leq \rho^Y(p,x).
    \]
\end{cor}

\begin{proof}
Let $a\in (\partial C)^\n_C$ be given by \Cref{cor:dist_to_cluster_attained_near_bdry} so that $d(a,x)\leq d(p,x)$. Since $a\in C$ and $x\not\in C$,
\[
\core^Y_k(a)\leq \epsilon\leq d(a,x).    
\] Therefore,
\[
\rho^Y(a,x) =\max\{\core^Y_k(x),d(a,x)\}\leq \max\{\core^Y_k(x),d(p,x)\}  \leq \rho^Y(p,x).  
\]
\end{proof}

We now consider a graph constructed using $F$. It depends on an increasing sequence $0 = \epsilon_0 < \epsilon_1 < \ldots$, which we now fix.
\begin{defn}\label{defn:frankie_stages}
Set $X_1=X$ and $X_{i+1} = F(X_i,\epsilon_i)$ for $i\geq 1$. 
Let $H(0)$ be the weighted graph with vertex set $X_1$ and no edges. For $i\geq 0$ define
\[
H(i+1) = \bigcup\limits_{C\in \D^*(X_{i+1},\epsilon_{i+1})}G(B_{\epsilon_{i+1}}(C),\rho^{B_{\epsilon_{i+1}}(C)})_{\epsilon_{i+1}} \bigcup H(i). 
\]    
\end{defn}

\begin{prop}\label{prop:successive_reach_dist}
Let $p_0,q_0\in X$. For every $i\in \N$, there exist $p_i,q_i\in J(X_i,\epsilon_i)$ such that
\begin{itemize}
    \item $p_{i-1} = p_i$ or $p_{i-1},p_i\in A_i$ for $A_i\in\D^*(X_i,\epsilon_i)$,
    \item $q_{i-1} = q_i$ or $q_{i-1},q_i\in B_i$ for $B_i\in\D^*(X_i,\epsilon_i)$, 
\end{itemize} and
\begin{equation}\label{eqn:successive_reach_dist}
\rho^{X_i}(p_i,q_i) \leq \rho^{X_{i-1}} (p_{i-1},q_{i-1}).    
\end{equation}
In particular, there exist $p_i,q_i\in J(X_i,\epsilon_i)$ with
\[
\rho^{X_{i+1}}(p_i,q_i)\leq \rho^X(p_0,q_0)    
\]
and paths in $H(i)$ that join $p_0$ to $p_i$ and $q_0$ to $q_i$.
\end{prop}

\begin{proof}
    Set $p_{-1} = p_0$, $q_{-1} = q_0$ and $X_{-1} = X_0 = X$. Note that with such definitions \cref{eqn:successive_reach_dist} is satisfied for $i=0$.

    Let $i\in \N$ and suppose that the conclusion holds for $i-1$ and let 
    \[
    p_{i-1},q_{i-1}\in J(X_{i-1},\epsilon_{i-1}) \subset F(X_{i-1},\epsilon_{i-1}) = X_i.
    \] By \Cref{lem:local_core_reach} applied to $p= p_{i-1}$, $q = q_{i-1}$, $\epsilon = \epsilon_{i-1}$ and $Y = X_{i-1}$, we have
    \begin{equation}\label{eqn:equal_reach}
    \rho^{X_i}(p_{i-1},q_{i-1})  = \rho^{X_{i-1}}(p_{i-1},q_{i-1}).
    \end{equation}
    There are four cases to consider. 
    
    First suppose that $p_{i-1},q_{i-1}\in \NN(X_i,\epsilon_i)\subset J(X_i,\epsilon_i)$. Set $p_i = p_{i-1}$ and $q_i=q_{i-1}$, so that \cref{eqn:equal_reach} gives \cref{eqn:successive_reach_dist} in this case.

    Now suppose that $p_{i-1}\not\in\NN(X_i,\epsilon_i)$ and $q_{i-1}\in\NN(X_i,\epsilon_i)$. Let $A_i \in \D^*(X_{i},\epsilon_{i})$ with $p_{i-1}\in A_{i}$ and set 
    \[
        q_i = q_{i-1}\in\NN(X_i,\epsilon_i)\subset J(X_i,\epsilon_i).
    \]  Let $p_i\in (\partial A_{i})^\n_{A_{i}}\subset J(X_i,\epsilon_i)$ be as in \Cref{cor:reach_dist_to_cluster_attained_near_bdry} applied to $Y = X_i$, $p = p_{i-1}$ and $x = q_{i-1} = q_i$ so that
    \[
    \rho^{X_i}(p_i,q_i)\leq \rho^{X_i}(p_{i-1},q_i) = \rho^{X_{i}}(p_{i-1},q_{i-1}).   
    \] Applying \cref{eqn:equal_reach} gives \cref{eqn:successive_reach_dist} as required.

    If $q_{i-1}\not\in\NN(X_i,\epsilon_i)$ and $p_{i-1}\in\NN(X_i,\epsilon_i)$, exchange $p_{i-1}$ and $q_{i-1}$ and apply the previous case.

    Finally, if $p_{i-1},q_{i-1}\not\in \NN(X_i,\epsilon_i)$, there exist $A_i,B_i\in \D^*(X_i,\epsilon_i)$ with $p_{i-1}\in A_i$ and $q_{i-1}\in B_i$. If $A_i=B_i$, setting $p_i = q_i\in (\partial A_i)_{A_i}$ suffices. Otherwise, (similarly to the previous two cases) by two applications of \Cref{cor:reach_dist_to_cluster_attained_near_bdry}, there exist $p_i\in (\partial A_i)^\n_{A_i}\subset J(X_i,\epsilon_i)$ such that 
    \[
    \rho^{X_i}(p_i,q_{i-1})\leq \rho^{X_i}(p_{i-1},q_{i-1}),    
    \] and $q_i\in (\partial B_i)^\n_{B_i}\subset J(X_i,\epsilon_i)$ such that
    \[
    \rho^{X_i}(p_i,q_i)\leq \rho^{X_i}(p_i,q_{i-1}).    
    \] 
By combining these two inequalities with \cref{eqn:equal_reach}, we obtain \cref{eqn:successive_reach_dist}.

For the in particular statement, since $p_i,q_i\in J(X_i,\epsilon_i)$, \cref{eqn:equal_reach} implies 
\[
\rho^{X_{i+1}}(p_i,q_i) = \rho^{X_i}(p_i,q_i)
\leq \rho^{X_{i-1}}(p_{i-1},q_{i-1})\leq \ldots \leq \rho^X(p_0,q_0).
\] 
Also, for any $1\leq j \leq i$, $A\in \D^*(X_j,\epsilon_j)$, and  $a,b\in A$, there is a path in $H(j)$ joining $a$ and $b$ (its vertices are contained in $A$). Since $E(H(j))\subset E(H(i))$, this path is contained in $H(i)$. 
Thus, there exists a path from $p_0$ to $p_i$ in $H(i)$. 
Similarly, there exists a path from $q_0$ to $q_{i}$ in $H(i)$. 
\end{proof}

Using the previous Proposition we are able to construct a graph whose connected components are identical to those produced by $\HDBS^*$ up to a given scale.
\begin{lem}\label{lem:edges}
Let $i\in \N$. For every $\epsilon_{i-1}<\alpha\leq\epsilon_i$,
\begin{enumerate}
    \item\label{lem:edges_1} $E(H(i)_{\alpha})\subset E(G(X,\rho^X)_\alpha)$
    \item\label{lem:edges_2} if $(p,q) \in E(G(X,\rho^X)_\alpha)$ then there is a path in $H(i)_\alpha$ joining $p$ and $q$.
\end{enumerate}
\end{lem}   

\begin{proof}
    First notice that, for any $m\in \N$ and $\epsilon_{m-1}<\alpha\leq\epsilon_m$,
\begin{equation}\label{eqn:slice_mixed_graph}
    H(m)_\alpha = \bigcup\limits_{C\in \D^*(X_{m},\epsilon_{m})}G(B_{\epsilon_{m}}(C),\rho^{B_{\epsilon_{m}}(C)})_{\alpha} \bigcup H(m-1).
\end{equation} 

We prove \Cref{lem:edges_1} by induction. For any $C\in\D^*(Y,\epsilon)$ and $p,q\in C$, we have $\core_k^{B_\epsilon(C)}(p) = core_k^Y(p)$ and $\rho^{B_\epsilon(C)}(p,q) = \rho^Y(p,q)$. This implies
\[
E(G(X,\rho^X)_\alpha) = \bigcup\limits_{C\in\D^*(X,\epsilon_1)}E(G(B_{\epsilon_1}(C),\rho^{B_{\epsilon_1}(C)})_{\alpha})
=: E(H(1)_\alpha)
\] proving the lemma when $i=1$.

Now assume that \cref{lem:edges_1} is satisfied for $i=m$ and let $\epsilon_{m}<\alpha\leq \epsilon_{m+1}$. Given $C\in\D^*(X_{m+1},\epsilon_{m+1})$, the containment
\[
    E(G(B_{\epsilon_{m+1}}(C),\rho^{B_{\epsilon_{m+1}}(C)})_{\alpha})\subset E(G(X,\rho^X)_\alpha)
\] follows from the fact that $\rho^X \leq \rho^A$ for any subset $A\subset X$. Also by the induction hypothesis, 
\[
E(H(m)_{\epsilon_m})\subset E(G(X,\rho^X)_{\epsilon_{m}})\subset E(G(X,\rho^X)_\alpha).
\] Thus \Cref{eqn:slice_mixed_graph} concludes the proof of \cref{lem:edges_1}.

To prove \cref{lem:edges_2} let $m\in \N$, $\epsilon_{m-1}<\alpha\leq\epsilon_{m}$ and $(p,q) \in E(G(X,\rho^X)_\alpha)$. Applying \Cref{prop:successive_reach_dist} with $p_0 =p$ and $q_0=q$, we obtain points $p_{m-1},q_{m-1}\in J(X_{m-1},\epsilon_{m-1})\subset X_m$ such that 
\[
\rho^{X_m}(p_{m-1},q_{m-1}) \leq \rho^X(p,q)\leq\alpha\leq\epsilon_m   
\] and two paths in $E(H(m-1))\subset E(H(m)_\alpha)$ joining $p_0$ and $p_{m-1}$ and $q_0$ and $q_{m-1}$.
In particular, $(p_{m-1},q_{m-1})\in E(G(B_{\epsilon_{m}}(C),\rho^{B_{\epsilon_{m}}(C)})_{\alpha})$ for some $C\in \D^*(X_m,\epsilon_m)$ and so, by \Cref{eqn:slice_mixed_graph}, $(p_{m-1},q_{m-1})\in E(H(m)_\alpha)$. Combining this edge and the two previous paths gives the required path from $p$ to $q$ in $H(m)_\alpha$.

\end{proof}

Finally we define a graph whose connected components coincide with those of $\HDBS^*$ at all scales.
\begin{defn}\label{defn:final_frankie}
For fixed $i \in \N$ define
\[
\mathcal{F}(i) = G(X_{i+1},\rho^{X_{i+1}})\cup H(i).
\]   
\end{defn}

\begin{thm}\label{thm:1} For any $i\in \N$, $\epsilon_{i-1}<\nu\leq\epsilon_i$, $\alpha\geq 0$ and any $p\neq q\in X$,
\begin{equation}\label{eqn:frankie_chunks}
p,q\in V(\Gamma) \text{ for } \Gamma\in\pi_0(H(i)_\nu) \Longleftrightarrow
p,q\in V(\tilde{\Gamma}) \text{ for } \tilde{\Gamma}\in\pi_0(G(X,\rho)_\nu).
\end{equation}
and
\begin{equation}\label{eqn:complete_Frankie_w_hdbscan}
p,q\in V(\Gamma) \text{ for } \Gamma\in\pi_0(\mathcal{F}(i)_\alpha) \Longleftrightarrow
p,q\in V(\tilde{\Gamma}) \text{ for } \tilde{\Gamma}\in\pi_0(G(X,\rho)_\alpha).
\end{equation}
\end{thm}

\begin{proof}
Fix $i\in\N$.
\Cref{eqn:frankie_chunks} follows immediately from \Cref{lem:edges}.

To prove \Cref{eqn:complete_Frankie_w_hdbscan} when $\alpha \leq \epsilon_i$, first notice that
\[
    E(G(X_{i+1},\rho^{X_{i+1}})_\alpha) \subset E(H(i)_\alpha).
\] 
Consequently, the result follows from \Cref{eqn:frankie_chunks}. 

Now suppose that $\alpha>\epsilon_i$ and note that 
\[
    \mathcal{F}(i)_\alpha = G(X_{i+1},\rho^{X_{i+1}})_\alpha\cup H(i).
\] Since $\rho^{X_{i+1}}\leq \rho^X$ wherever they are both defined,
\[
E(G(X_{i+1},\rho^{X_{i+1}})_\alpha)\subset E(G(X,\rho^X)_\alpha).
\] By \Cref{lem:edges} \cref{lem:edges_1}, 
\[
E(H(i)_\alpha)\subset E(G(X,\rho^X)_\alpha).
\] Therefore 
\[
E(\mathcal{F}_\alpha) \subset E(G(X,\rho^X)_\alpha)    
\] and so the first implication in \Cref{eqn:complete_Frankie_w_hdbscan} holds.

Finally, assume $(p_0,q_0)\in E(G(X,\rho^X)_\alpha)$. Let $p_i,q_i\in J(X_i,\epsilon_i)\subset X_{i+1}$ be as in \Cref{prop:successive_reach_dist} so that
\[
\rho^{X_{i+1}}(p_i,q_i)\leq \rho^X(p_0,q_0)\leq \alpha.    
\] In particular, $(p_i,q_i)\in E(G(X_{i+1},\rho^{X_{i+1}})_\alpha)$. Since there exist paths in $H(i)$ connecting $p_0$ to $p_i$ and $q_0$ to $q_i$, there exists a path from $p_0$ to $q_0$ in $\mathcal{F}(i)_\alpha$. Consequently, the second implication in \Cref{eqn:complete_Frankie_w_hdbscan} holds.

\end{proof}

\section{Example: \texorpdfstring{$\SDBS^*$}{S-DBSCAN*} to cluster building data of the US}\label{sec:applications_QDBS}

To illustrate our methods, we cluster datasets of building footprints in the United States. 
We use the Microsoft Open Buildings dataset \cite{MicrosoftBuildingData} which is freely available. 

We cluster two different datasets using $\SDBS^*$ and sklearn.$\DBS^*$ and compare the runtime in each instance. 
Our first dataset corresponds to the state of Utah which consists of 1,004,734 data points. 
The second dataset corresponds to the entire United States and has size 124,828,547. 
For each of these datasets, we produce the $\DBS^*$ clusters for $\epsilon =$ 3,000 and $\epsilon = $12,000, each with $k = $1,900. 
In each instance, the run time is recorded in seconds.

These experiments were performed using an Intel Xeon CPU running at 2.30GHz with 64GB of memory on 64-bit Ubuntu 20.04.3 LTS.

Tables \ref{tab:SDBSCAN_utah} and \ref{tab:SDBSCAN_usa} show the runtime of $\SDBS^*$ and sklearn.$\DBS^*$, the resulting number of core points $|\C|$, number of clusters $|\D^*|$ and maximum and mean cluster sizes. 
It should be noted that the implementation of $\SDBS^*$ is not optimized and relies purely on python code which negatively affects performance. On the other hand, sklearn.$\DBS^*$ is optimized for speed and not memory efficiency. Five workers are used for all runs of the $\SDBS^*$ algorithm.

\begin{table}[ht!]

\begin{tabular}{|c|c|c|c|c|c|c|c|}
\hline
& $\epsilon$ & Time & $|\C|$ & $|\D^*|$ & $\max|C|$ & mean $|C|$ \\ \hline
$\SDBS^*$        & 2,000    & 34                 & 670,611      & 30       & 305,024           & 22,354           \\ \hline
$\SDBS^*$        & 8,000    & 6                  & 877,112      & 19       & 741,710           & 46,164 \\ \hline
$\DBS^*$ & 2,000    & 91                 & 670,611      & 30       & 305,024           & 22,354           \\ \hline
$\DBS^*$ & 8,000   & -                  & -           & -        & -                & -                 \\ \hline
\end{tabular}
\vspace{0.2cm}
\caption{$\DBS^*$ for Utah}
\label{tab:SDBSCAN_utah}
\end{table}

\begin{table}[ht!]
\begin{tabular}{|c|c|c|c|c|c|c|}
\hline
& $\epsilon$ & Time & $|\C|$ & $|\D^*|$ & $\max|C|$ & mean $|C|$\\ \hline
$\SDBS^*$        & 3,000    & 1,752               & 79,919,086    & 3,244     & 6,237,660          & 24,636        \\ \hline
$\SDBS^*$        & 12,000   & 961                & 120,976,900   & 429      & 94,517,493         & 281,997          \\ \hline
$\DBS^*$ & 3,000    & -                  & -           & -        & -                & -                 \\ \hline
$\DBS^*$ & 12,000   & -                  & -           & -        & -                & -                 \\ \hline
\end{tabular}
\vspace{0.2cm}
\caption{$\DBS^*$ for USA}
\label{tab:SDBSCAN_usa}
\end{table}

In \Cref{tab:SDBSCAN_utah} it can be seen that, for the state of Utah, $\SDBS^*$ has a lower runtime than sklearn.$\DBS^*$ for $\epsilon =2,000$ .  
Furthermore, for $\epsilon = $ 8,000 sklearn.$\DBS^*$ ran out of memory. In contrast, the runtime of $\SDBS^*$ \emph{decreases} for the larger value of $\epsilon$.
This is because more cubes of the partition become classified as dense, and consequently, the required number of distance calculations significantly decreases.

\Cref{tab:SDBSCAN_usa} shows the running time of $\SDBS^*$ and sklearn.$\DBS^*$ when clustering building data of the US. 
For all parameter choices sklearn.$\DBS^*$ runs out of memory. 
Similarly to the Utah dataset, for the larger value of $\epsilon$, the $\SDBS^*$ runtime decreases. 

\section{Example: \texorpdfstring{$\SHDBS^*$}{S-HDBSCAN*} to cluster building data of the US}\label{sec:applications_QHDBS}

Recall that, for a fixed value of $\epsilon$, the connected components of $\DBS^*$ (with more than one element) coincide with the connected components of $G(X,\rho)_\epsilon$, used to construct the $\HDBS^*$ clusters, see \cref{eqn:dbs_components_equal_hdbs_components}. 
This fact allows us to construct $\HDBS^*$ in steps by first separating the data into $\DBS^*$ clusters and then obtaining the connected components of $G(X,\rho)_\epsilon$ restricted to each $\DBS^*$ cluster. 

More explicitly the steps of $\SHDBS^*$ are:
\begin{enumerate}
\item Cluster $X_1 = X $ using $\SDBS^*$ with an initial choice of $\epsilon_1$. 
\item For each cluster $C\in \D^*(X_1,\epsilon_1)$, we calculate the reachability distance $\tilde{\rho}$ using \emph{only} points in $C^\m$ and obtain the weighted graph $G(C,\tilde{\rho})_{\epsilon_1}$. 
\item Combine an extended boundary of each cluster and $\NN(X_1,\epsilon_1)$ to form $X_2$. Repeat steps 1 and 2 on $X_2$ for a choice of $\epsilon_2 > \epsilon_1$. 
\item Combine an extended boundary of each cluster in $\D^*(X_2,\epsilon_2)$ and $\NN(X_2,\epsilon_2)$ to form $X_3$.
\item Calculate $\tilde{\rho}$ in $X_3$ and the weighted graph $G(X_3,\tilde{\rho})$.
\item Combine all weighted graphs produced at each step to produce our final graph $G$. \Cref{thm:1} guarantees that, for any $\epsilon > 0$, the connected components of $G_\epsilon$ agree with those of $G(X,\rho)_\epsilon$.
\end{enumerate}
This procedure can be iterated any number of times, we use $3$ for illustration only. 
At each step, each $\SDBS^*$ cluster is processed independently and in parallel.

One of the fastest implementations of $\HDBS^*$ is sklearn.$\HDBS^*$, \cite{McInnes17}. 
The speed of this algorithm comes from combining several high performance algorithms used to optimise the steps of $\HDBS^*$ with high time complexity. 
A central tool used in these optimisations are \emph{$k$-d trees}, which are state-of-the-art in nearest neighbour type searches. 

A $k$-d tree partitions $\R^n$ into non-overlapping regions, indexed by a given \emph{reference tree}  $\tau_r\subset X\subset\R^n$. 
Given a \emph{query point} $p\in\tau_q\subset X$, one can use the partition to efficiently find the nearest neighbours in $\tau_r$ to $p$. 
To further optimise performance one can build a second partition based on the \emph{query tree} $\tau_q$.

In \cite{McInnes17}, query and reference trees are used to calculate $\core_k$ and a \emph{minimum spanning tree}. 
In practice, given a weighted graph $H$, to obtain the connected components of $H_\epsilon$, for every $\epsilon$, a standard and efficient technique is to store a minimum spanning tree of $H$. This is done using the dual tree Borůvka algorithm. 

The performance and memory usage of calculating $\core_k$ and the Borůvka algorithm worsens as the size of the reference and query trees increase. 
In fact, in \cite{hbscan_benchmarking} it can be seen that clustering very large datasets using sklearn.$\HDBS^*$ becomes infeasible. 
An underlying reason for this is that these trees are constructed using the entire dataset. 
A common solution is to process the points in the query tree in batches. 
However, given the size of our dataset and our choice of $k$, the reference trees are simply too large and batch processing the query trees alone does not make the task tangible.

To solve this issue $\SHDBS^*$ partitions $X$ to create many smaller trees. 
The difficulty lies in partitioning the reference trees whilst ensuring that the correct $\core_k$ and the correct weights of a spanning graph are produced. 
Since $\DBS^*$ clusters naturally partition $X$ they can be used to achieve this. 
We take this further and iteratively use $\DBS^*$ clusters to incrementally construct a spanning graph. 
After each iteration we can remove points from the dataset that are now redundant, namely the interiors of clusters of previous iterations. 
Moreover, the calculations for each cluster are independent and are processed in parallel. 
A crucial point is that, the minimum spanning trees of each cluster are combined into a graph with a relatively small amount of edges that, by \Cref{thm:1}, allows us to recover the connected components required for $\HDBS^*$.  

In practice, from one iteration to the next we only add edges $(p,q)$ with weight $\epsilon_{i-1}< \omega(p,q)\leq \epsilon_i$ because this is sufficient to complete a small spanning graph. 
Further, we only add an edge $(p,q)$ if $p$ and $q$ were not in the same connected component in a previous iteration.

We also note that $\SHDBS^*$ can be used recursively: if a particular cluster cannot be processed directly, the same approach can be applied to the cluster.

To demonstrate our theory, we built an implementation of $\SHDBS^*$ and ran it on the same example datasets than \Cref{sec:applications_QDBS} (using the same computer). 
Since this was for demonstration purposes only it consists of unoptimised Python scripts. As a point of comparison, we also ran sklearn.$\HDBS^*$.

The results are shown in \Cref{tab:hdbscans}, with time measured in seconds. 
The three final columns refer to the final $\HDBS^*$ clusters, after the hierarchical clusters have been scored across all  scales. 
In the case of Utah, where sklearn.$\HDBS^*$ completes, we see that $\SHDBS^*$ has a comparable running time. 
However for the entire dataset, sklearn.$\HDBS^*$ does not terminate, whereas $\SHDBS^*$ does. 

\begin{table}[ht!]
\begin{tabular}{|c|c|c|c|c|c|}
\hline
& & Time &  $|\mathcal{H}^*|$ & $\max|C|$ & mean $|C|$ \\ \hline
$\SHDBS^*$ & Utah   & 649                 & 40       & 307,149           & 20,386         \\ \hline
$\HDBS^*$ & Utah    & 415    & 40       & 307,149           & 20,386         \\ \hline

$\SHDBS^*$ & USA             & 89,670               & 5,222     & 44,515,947         & 23,904         \\ \hline
$\HDBS^*$ & USA   & -                & -                 & -                & -                 \\ \hline
\end{tabular}
\vspace{0.2cm}
\caption{Final $\HDBS^*$ clusters, $k = m = 1900$}
\label{tab:hdbscans}
\end{table}

\subsection{Tree sizes at each iteration}
We now illustrate the reduction of the query and reference trees, at each iteration, granted by our theory. 
For the first iteration, we use $\epsilon_1 = 3,000$ and $\SDBS^*$ produces 3,244 clusters of the entire dataset $|X_1| =124,828,547$. 
For the second iteration we use $\epsilon_2 = 12,000$ and $\SDBS^*$ produces 429 clusters of $|X_2| = 110,943,965$. 
The last step of the process uses $|X_3| = 66,145,965$. We refer the reader to \cref{tab:SDBSCAN_usa} for details of the $\DBS^*$ clusters of $X_1$, relevant to iterations 1 and 2.

Since our datasets are contained in $\R^2$, we have $\n =1$, $\m = 3$ and $\nn = 4$. 
\Cref{tab:first_iter} contains the size distribution of the query tree $\tau_q(C)=C$ and reference tree $\tau_r(C)$ of clusters $C\in\D^*(X_1,\epsilon_1)$ used to calculate $\core_k$. 
By \Cref{obs:epsilon_neighood_extension}, it suffices to take $\tau_r(C)=C^3$ as a reference tree to calculate $\core_k(p)$ for points $p\in C$. 
The query and reference tree required to build the minimum spanning tree of $C$ is $C$. \Cref{tab:first_iteration_tree_sizes_largest_clusters} contains the sizes of these trees for the 10 largest clusters. 

\begin{table}[ht!]
    \begin{tabular}{|c|c|c|} 
    \hline
     & $\tau_q(C)$ & $\tau_r(C)$\\ \hline
    mean  & 24,636        & 29,059            \\ 
    std   & 179,487      & 190,162     \\ 
    min   & 1            & 1,956       \\ 
    25\%  & 1,908        & 3,702      \\ 
    50\%  & 3,204        & 5,692        \\ 
    75\%  & 7,080        & 10,742      \\ 
    max   & 6,237,660    & 6,745,531   \\ \hline
    count & \multicolumn{2}{|c|}{3,244}\\ \hline
    \end{tabular}
    \vspace{0.2cm}
    \caption{First iteration: query and reference tree sizes.}
    \label{tab:first_iter}
    \end{table}

\begin{table}[ht!]
    \begin{tabular}{|c|c|c|c|}
    \hline 
    $\tau_q(C)$ & $\tau_r(C)$ & $\tau_q(C)$ & $\tau_r(C)$ \\ \hline
6,237,660   & 6,745,531   & 5.00\%      & 5.40\%      \\
4,055,724   & 4,090,661   & 3.25\%      & 3.28\%      \\
3,089,031   & 3,204,504   & 2.47\%      & 2.57\%      \\
2,162,629   & 2,285,680   & 1.73\%      & 1.83\%      \\
1,975,669   & 2,202,149   & 1.58\%      & 1.76\%      \\
1,721,895   & 1,814,330   & 1.38\%      & 1.45\%      \\
1,680,441   & 1,760,533   & 1.35\%      & 1.41\%      \\
1,661,300   & 1,740,421   & 1.33\%      & 1.39\%      \\
1,453,234   & 1,574,444   & 1.16\%      & 1.26\%      \\
1,450,173   & 1,556,962   & 1.16\%      & 1.25\%      \\\hline
25,487,756  & 26,975,215  & 20.42\%     & 21.61\%    \\\hline
\end{tabular}
    \vspace{0.2cm}
    \caption{First iteration: tree sizes of largest clusters and their percentage of the whole dataset.}
    \label{tab:first_iteration_tree_sizes_largest_clusters}
\end{table}


This iteration of $\SHDBS^*$ processes $\core_k$ and the minimum spanning tree of approximately 80 million points, see \Cref{tab:SDBSCAN_usa}. \Cref{tab:first_iter,tab:first_iteration_tree_sizes_largest_clusters} show that this is achieved using reference trees with less than 7 million points, 5\% of the entire dataset. 
In fact, about 55 million of these points (44\% of the dataset) are processed with query and reference trees smaller than 1.5 million points. \Cref{tab:first_iter} shows that 75\% of the clusters have reference trees with less than 11,000 points, four orders of magnitude smaller than the entire dataset.

This massive reduction of the sizes of these reference trees is what makes it possible for us to cluster a dataset of this size, which we do in 24hrs. In data of this magnitude, it is appropriate to model the complexity of sklearn.$\HDBS^*$ algorithm as $O(n^2)$, \cite{hbscan_benchmarking}. Under this assumption, dividing the dataset into such small pieces is a significant improvement.

There are two possible ways to continue constructing a spanning graph from scales $\epsilon_1$ to $\epsilon_2$. The first is to repeat the first iteration with $\epsilon_2$ but only processing noise points from the first iteration. That is, for each $C\in\D^*(X_1,\epsilon_2)$ the trees  $\tau^c_q(C)=C\cap\NN(X_1,\epsilon_1)$ and $\tau^c_r(C) = C^3$ calculate $\core_k$. The tree $\tau^m = C$ is used as a query and reference tree to calculate the minimum spanning tree. The sizes of these trees for the 10 largest clusters are contained in the first three columns of \Cref{tab:second_iteration_tree_sizes_largest_clusters}.  \Cref{tab:second_iteration_percent_tree_sizes_largest_clusters} shows the same values as a percentage of the size of the entire dataset.

We consider an alternative way that reduces the sizes of the trees even further by first removing an interior from each cluster of the first iteration. More precisely, set $X_2 = F(X_1,\epsilon_1)$ as in \Cref{defn:good_set_reach} (keeping only the 4th extension of the boundary of each cluster) and, for each $C\in\D^*(X_2,\epsilon_2)$, set $\tau^c_q(C) = C\cap \NN(X_1,\epsilon_1)$ and $\tau^c_r(C) = C^3$. To calculate the minimum spanning tree, the query and reference trees agree and equal $\tau^m(C) = C\cap J(X_1,\epsilon_1)$ (as defined in \Cref{defn:good_set_reach}, keeping only the 1st extension of the boundary of previous iteration clusters). The sizes of these trees for the 10 largest clusters are shown in columns 1, 4 and 5 of \Cref{tab:second_iteration_tree_sizes_largest_clusters,tab:second_iteration_percent_tree_sizes_largest_clusters}.

By \Cref{thm:1} and \Cref{cor:reach_dist_to_cluster_attained_near_bdry}, the spanning graphs produced by either of these methods are equivalent to $G(X,\rho)_{\epsilon}$ for any $\epsilon\leq \epsilon_2$. 

\begin{table}[ht!]
    \begin{tabular}{|c|c|c|c|c|c|}
        \hline
    Both          & \multicolumn{2}{|c|}{With interior}     & \multicolumn{2}{|c|}{Without} \\ \hline
    $\tau^c_q(C)$ & $\tau^c_r(C)$ & $\tau^m(C)$ & $\tau^c_r(C)$ & $\tau^m(C)$ \\ \hline
    36,039,363    & 95,728,364    & 94,517,493  & 85,297,674    & 63,373,014  \\
    984,007       & 10,580,949    & 10,495,579  & 8,697,199     & 4,699,615   \\
    806,604       & 3,575,609     & 3,530,441   & 3,157,678     & 2,039,457   \\
    242,434       & 1,553,354     & 1,519,012   & 1,241,211     & 740,889     \\
    208,895       & 1,530,487     & 1,518,726   & 1,157,362     & 537,396     \\
    170,270       & 813,466       & 797,390     & 740,581       & 425,973     \\
    105,852       & 578,156       & 576,816     & 530,175       & 362,543     \\
    89,095        & 545,200       & 498,704     & 477,756       & 356,340     \\
    83,919        & 477,756       & 443,445     & 391,985       & 231,348     \\
    67,557        & 454,799       & 441,073     & 369,047       & 222,249     \\ \hline
    38,797,996    & 115,838,140   & 114,338,679 & 102,060,668   & 72,988,824 \\ \hline
    \end{tabular}
    \vspace{0.2cm}
    \caption{Second iteration: tree sizes of largest clusters}
    \label{tab:second_iteration_tree_sizes_largest_clusters}
    \end{table}

\begin{table}[ht!]
        \begin{tabular}{|c|c|c|c|c|c|} \hline
        Both          & \multicolumn{2}{|c|}{With interior}     & \multicolumn{2}{|c|}{Without} \\ \hline
        $\tau^c_q(C)$ & $\tau^c_r(C)$ & $\tau^m(C)$ & $\tau^c_r(C)$ & $\tau^m(C)$ \\ \hline
        28.87\%       & 76.69\%       & 75.72\%     & 68.33\%       & 50.77\%     \\
        0.79\%        & 8.48\%        & 8.41\%      & 6.97\%        & 3.76\%      \\
        0.65\%        & 2.86\%        & 2.83\%      & 2.53\%        & 1.63\%      \\
        0.19\%        & 1.24\%        & 1.22\%      & 0.99\%        & 0.59\%      \\
        0.17\%        & 1.23\%        & 1.22\%      & 0.93\%        & 0.43\%      \\
        0.14\%        & 0.65\%        & 0.64\%      & 0.59\%        & 0.34\%      \\
        0.08\%        & 0.46\%        & 0.46\%      & 0.42\%        & 0.29\%      \\
        0.07\%        & 0.44\%        & 0.40\%      & 0.38\%        & 0.29\%      \\
        0.07\%        & 0.38\%        & 0.36\%      & 0.31\%        & 0.19\%      \\
        0.05\%        & 0.36\%        & 0.35\%      & 0.30\%        & 0.18\%      \\ \hline
        31.08\%       & 92.80\%       & 91.60\%     & 81.76\%       & 58.47\%    \\ \hline
        \end{tabular}
\vspace{0.2cm}
    \caption{Second iteration: percentage of total dataset of tree sizes of largest clusters}
    \label{tab:second_iteration_percent_tree_sizes_largest_clusters}
\end{table}

In this second iteration there is a very large cluster containing 95 million points, which has large reference and query trees as a result. However, this cluster consists only of 37 million new points and, at most, 3,244 clusters already formed in the iteration 1. Consequently, our calculations of $\core_k$ restrict to the new 37 million points and the Borůvka algorithm only requires 37 million iterations, instead of 95 million, to complete the spanning graph of this cluster.
We can see in \Cref{tab:second_iteration_tree_sizes_largest_clusters,tab:second_iteration_percent_tree_sizes_largest_clusters} that the trees corresponding to the remaining 428 $\DBS^*$ clusters have a comparable distribution to those of the first iteration. 

\Cref{tab:second_iteration_percent_tree_sizes_largest_clusters} shows that removing the interior of the clusters from iteration 1 reduces the size of the reference trees by about 10\% of the total dataset when calculating $\core_k$ and by 30\% when calculating the spanning graphs.

For the final iteration, we calculate core distances with a single query tree $\tau^c_q = \NN(X_1,\epsilon_2)$. If we keep the interiors of clusters from the previous iterations, we must use $\tau^c_r =  X_1$ as a reference tree. However, we can reduce the reference tree significantly by removing the interiors of clusters giving $\tau^c_r =  X_3$. For the minimum spanning tree we have query and reference trees equal to $\tau^m =X_1$ in the first case and $\tau^m = J(X_2,\epsilon_2)$ in the second case.

\begin{table}[ht!]
\begin{tabular}{|c|c|c|c|c|}\hline
    Both          & \multicolumn{2}{|c|}{With interior}     & \multicolumn{2}{|c|}{Without} \\ \hline
$\tau^c_q$ & $\tau^c_r$ & $\tau^m$ & $\tau^c_r$ & $\tau^m$ \\\hline
3,851,647     & 124,828,547   & 124,828,548 & 66,145,965    & 28,931,402  \\
3.09\%        & 100\%      & 100\%    & 52.99\%       & 23.18\%   \\ \hline
\end{tabular}
\vspace{0.2cm}
\caption{Final iteration: tree sizes and their percentage of the whole dataset.}
\label{tab:third_iteration_tree_sizes_largest_clusters}
\end{table}

This iteration highlights the importance of removing the interiors of previous iterations. In \Cref{tab:third_iteration_tree_sizes_largest_clusters} we see that calculating $\core_k$ has a relatively small query tree and that the impact of removing the interiors of previous iterations is significant, reducing the size of the reference tree by almost 50\%. For calculating the spanning graph, the reduction from removing the interiors is even greater and reduces the size of both the query and the reference trees by almost 80\% of the total dataset.

\subsection{Final remarks}

There is a lot of potential to implement $\SHDBS^*$ in parallel and to employ sophisticated algorithms that optimise its performance. Moreover, the number of points that we removed in this final iteration motivates further study in how to optimally choose the $\epsilon_i$ to have a greater impact by removing larger interiors in each iteration.

\bibliography{references}\bibliographystyle{abbrvurl}

\end{document}